\def\year{2020}\relax
\documentclass[letterpaper]{article} 
\usepackage{aaai20}  
\usepackage{times}  
\usepackage{helvet} 
\usepackage{courier}  
\usepackage[hyphens]{url}  
\usepackage{graphicx} 
\usepackage{graphicx}  
\usepackage{microtype}
\usepackage{amsmath} 
\usepackage{epstopdf} 
\usepackage{amsthm}
\usepackage{dsfont}
\usepackage{amsfonts}
\usepackage{multirow}
\usepackage{array}
\usepackage{bm}
\usepackage{algorithm}
\usepackage[noend]{algpseudocode}
\usepackage{subcaption}

\newtheorem{theorem}{Theorem}

\newtheorem{definition}{Definition}

\def\u{{\bf u}}
\def\v{{\bf v}}

\def\b{{\bf b}}

\def\x{{\bf x}}

\def\c{{\bf c}}

\DeclareMathOperator*{\argmax}{argmax}

\title{Proximity Preserving Binary Code \\ Using Signed Graph-Cut}
\author{
Inbal Lavi,\textsuperscript{\rm 1}
Shai Avidan,\textsuperscript{\rm 1} 
Yoram Singer\textsuperscript{\rm 2} and
Yacov Hel-Or,\textsuperscript{\rm 3}
 \\
\textsuperscript{\rm 1}Department of Electrical Engineering, Tel-Aviv University, Israel \\
\textsuperscript{\rm 2}Department of Computer Science, Princeton University, NJ, USA \\
\textsuperscript{\rm 3}School of Computer Science, The Interdisciplinary Center, Israel\\
inballavi@mail.tau.ac.il, toky@idc.ac.il
}

\begin{document}

\maketitle

\begin{abstract}
We introduce a binary embedding framework, called {\em Proximity Preserving Code (PPC)}, which learns similarity and dissimilarity between data points to create a compact and affinity-preserving binary code.  This code can be used to apply fast and memory-efficient approximation to nearest-neighbor searches. Our framework is flexible, enabling different proximity definitions between data points. In contrast to previous methods that extract binary codes based on unsigned graph partitioning, our system models the attractive and repulsive forces in the data by incorporating positive and negative graph weights. The proposed framework is shown to boil down to finding the minimal cut of a signed graph, a problem known to be NP-hard. We offer an efficient approximation and achieve superior results by constructing the code bit after bit. We show that the proposed approximation is superior to the commonly used spectral methods with respect to both accuracy and complexity. Thus, it is useful for many other problems that can be translated into signed graph cut. 
\end{abstract}

\section{Introduction}

Content-based image retrieval is a fundamental problem in computer vision, media indexing, and data analysis. A common solution to the problem consists of assigning each image an indicative feature vector and retrieving similar images by defining a distance metric in the feature vector space.

One of the successful uses of deep learning is {\em data embedding}
\cite{chopra2005learning,koch2015siamese}, where a network is used to map input data into a feature vector space, satisfying some desired distance properties. This technique has many applications, such as word embedding for machine translation \cite{mikolov2013efficient},
face embedding for identity recognition \cite{taigman2014deepface,facenet,wen2016discriminative,liu2017sphereface}, and  many more.  The main idea behind data embedding is to find a mapping from input space into a vector space where the distances in the embedding space conform with the desired task. 


In a typical scenario, the embedding space is several hundreds of bytes long (e.g., 512 bytes in FaceNet \cite{facenet} embedding), and a new query may be compared to the existing images by nearest-neighbor (NN) search. 
As the number of images scales up, the memory required to store all the examples becomes too large, and the time complexity to apply NN search becomes a critical bottleneck.  

Many solutions have been proposed to mitigate this issue,  including dimensionality reduction \cite{pca} and approximate NN search \cite{flann}.
In recent years, a family of algorithms called {\em Binary Hashing} or {\em Hamming Embedding} has gained popularity. These algorithms find a mapping from a feature space into a Hamming space using a variety of methods. The main advantages of a binary representation are the significant reduction in storage and in the time required to apply vector comparisons: vectors are compared not in high-dimensional Euclidean space, but rather in the Hamming space, utilizing the extremely fast XOR operation. 
This representation is highly valuable in mobile systems, as on-device training is limited due to computational shortage. This requires ad-hoc hashing methods that can be computed on simple hardware, and that can be generalized well to novel data points.

Many modern Hamming embedding techniques are data-dependent. Data-dependent methods work by learning an affinity matrix between data points while attempting to preserve their affinities in Hamming space. {\em In-sampled} techniques aim at generating a set of binary codes, a single code for each data point, whose Hamming distances conform with the affinity matrix. {\em Out-of-sample} techniques deal with novel samples that are not known in advance. These techniques learn a general functional mapping that maps query points from feature space into Hamming space.

Affinity between data pairs can be, for example, related to the metric distances between their associated features, or semantic relations indicating data points belonging to the same semantic class. The affinity matrix is usually relaxed to {\em positive} values, where small values indicate weak proximity (far pairs), and large values indicate strong proximity (near pairs). This encourages near pairs to be located close by in the Hamming space but does not constrain the far pairs.

We propose a binary hashing framework called {\em Proximity Preserving Code} (PPC). The main contribution of our method is that the binary code is constructed based on positive and negative proximity values, representing attractive and repulsive forces. These forces properly arrange the points in the Hamming space while respecting the pairwise affinities. Our solution models this proximity as a signed graph, and the code is computed by finding the min-cut of the graph. This problem can be formulated as the {\em max-cut} problem (due to the negative values) and is known to be NP-hard \cite{alon2004approximating}.
We demonstrate that our approach is more accurate and memory-efficient as compared to state of the art graph-based embeddings.

\section{Previous Works}
Previous works in Hamming embedding can be classified into two distinct categories: data-independent and data-dependent. Data-independent methods are composed of various techniques for dimensionality reduction or techniques for dividing the N-dimensional space into buckets with equal distributions. One of the most popular data-independent hashing methods is Locality Sensitive Hashing (LSH) \cite{lsh}. LSH is a family of hash functions that map similar items to the same hash bucket with higher probability than dissimilar items. 

Data-dependent methods learn the distribution of the data in order to create accurate and efficient mapping functions. These functions are usually comprised of three elements: the hash function, the similarity measure, and an optimization criterion. Hash functions vary and include linear functions \cite{mlh}, nearest vector assignment \cite{kmh}, kernel functions \cite{bre}, neural networks \cite{lin2015deep}, and more. Similarity measures include Hamming distance and different variants of Euclidean or other compute-intensive distances that are precomputed for vector assignment \cite{pq}. Optimization criteria mainly use variants of similarity preservation and code balancing. We will focus on binary hashing methods.

An influential work in binary hashing methods is Spectral Hashing \cite{sh}. 
This method creates code words $\{\c_i\}$ that preserve the data similarity. 
By defining an affinity matrix
\(W_{ij}=exp(-\frac{{\|x_i-x_j\|}^2}{\epsilon^2})\), the authors turn the hashing problems into a minimization of 
\(\sum_{ij} W_{ij} \|\c_i-\c_j\|^2\), subject to: 
$\c_i\in\{-1,1\}^k$, $\sum_i \c_i={\bf 0}$
(code balancing), and $\frac{1}{n}\sum_i \c_i \c_i^T=I$ (independence). 
This minimization problem for a single bit can be cast as a graph partitioning problem which is known to be NP-hard. A good approximation for this problem is achieved by using spectral methods. The code is obtained by computing the $k$ eigenvectors corresponding to the smallest eigenvalues of the graph Laplacian of $W$ and thresholding them at zero.

\citeauthor{agh} \shortcite{agh} proposed the Anchor Graph Hashing, a hashing method utilizing the idea of a low-rank matrix that approximates the affinity matrix, to allow a graph Laplacian solution that is scalable both for training and out-of-sample computation. 
\citeauthor{imh} \shortcite{imh} present Inductive Manifold Hashing, a method that learns a manifold from the data and utilizes it to find a Hamming embedding. They demonstrate their results with several approaches, including Laplacian eigen-maps and t-SNE.
\citeauthor{sdh} \shortcite{sdh} and \citeauthor{dgh} \shortcite{dgh} directly optimize the discrete problem, and employ discrete coordinate descent to achieve better precision on the graph problem.
Scalable Graph Hashing with Feature Transformation \cite{sgh} uses a feature transformation method to approximate the affinity graph, allowing faster computation on large scale datasets. They also proposed a sequential approach to learn the code bit-by-bit, allowing for error-correcting of the previously computed bits.
\citeauthor{lghsr} \shortcite{lghsr} revisit the spectral solution to the Laplacian graph and propose a spectral rotation that improves the accuracy of the solutions.

All of the above approaches formulate the graph Laplacian by defining an affinity matrix that takes into account the similarities between points in the training set. However, they do not address the dis-similarity, or push-pull forces in the data set. 
In this paper, we propose a binary embedding method that employs an affinity matrix of both positive and negative values. We argue that this type of affinity better represents the relationships between data points, allowing a more accurate code generation. 
The characteristics and the advantages of this work are as follows:

\begin{itemize}
\item Our code is constructed by solving a signed graph-cut problem, to which we propose a highly accurate solution. We demonstrate that the signed graph provides a better encoding for the forces existing in the coding optimization. We show that the commonly used spectral solution, which works well in the unsigned graph-cut problems, is unnecessary, costly, and inferior in this scenario.
\item The code is computed one bit at a time, allowing for error correction during the construction of the hashing functions.  
\item We split the optimization into two steps. We first optimize for a binary vector representing the in-sample data, and then we fit the hashing functions to obtain accurate code for out-of-sample points. 

\item Our framework is flexible, allowing various proximity definitions, including semantic proximity. This can be useful for many applications, especially in low computation environments.
\end{itemize}

\section{Problem Formulation}

\begin{figure*}[htb]
\begin{subfigure}[b]{0.67\textwidth}
\includegraphics[width=0.9\linewidth]{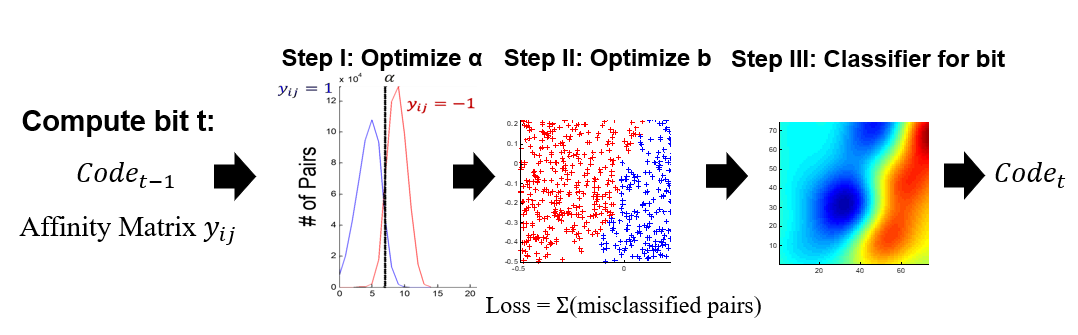}\caption{PPC Algorithm \label{fig:algo}}
\end{subfigure}
\begin{subfigure}[b]{0.3\textwidth}
\includegraphics[width=0.9\linewidth]{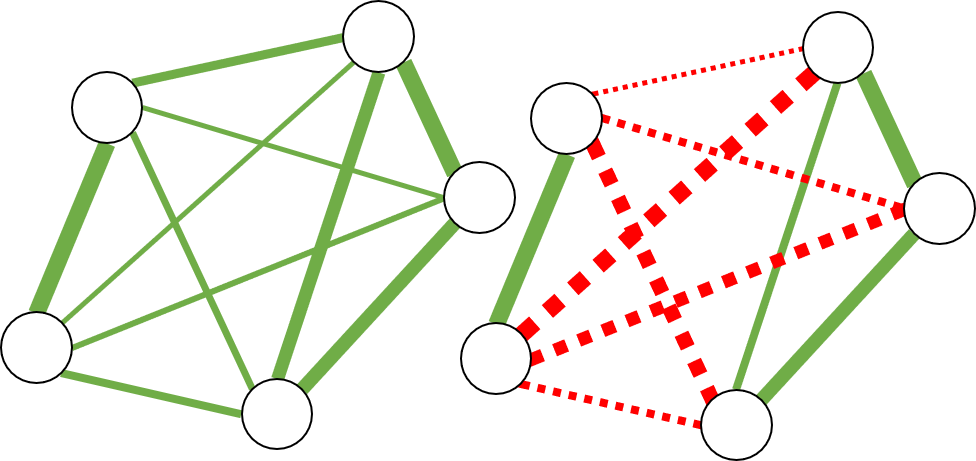}\caption{Unsigned vs. signed graph}
\end{subfigure}
\caption{(a) PPC algorithm overview. To compute the $t^{th}$ bit, we first find the optimal $\alpha$ for the existing code ($t-1$ bits), then compute a new bit $\b$ by minimizing the loss $\mathcal{E}$. For this bit, we compute a binary classifier $h^t(\x)$ that is used as the $t^{th}$ hashing function. (b) Illustration of an unsigned graph (left) vs. a signed graph (right) describing the same relations between nodes. The graph edges in green (solid line) are edges with positive weights, and the red (dashed lines) are edges with negative weights. The line thickness indicates the weight magnitude. }
\end{figure*}

We are given a set of $n$ data points $\mathcal{X}=\left \{\x_1, \x_2, \cdots, \x_n \right \}$, $\x_i \in \mathbb{R}^d$, in some
vector space, and a proximity relation between pairs of points $(i,j)\in \mathcal{S}$, where $\mathcal{S}=\{1..n\} \times \{1..n\}$. We assign each pair of points in $\mathcal{S}$ to be in the {\em Near} or {\em Far} group, according to some proximity measure. This proximity measure can have a semantic meaning, geometric meaning, or any other adjacency relation. Formally, we define \[\mathcal{N}=\left \{(i,j)~|~ \x_i~\mbox{and}~\x_j~\mbox{are in the same class} \right \}\] and \[\mathcal{F}=\left \{(i,j)~|~ \x_i~\mbox{and}~\x_j~\mbox{are in different classes} \right \}\] 
Note that $\mathcal{N}$ and $\mathcal{F}$ induce a partition
of $\mathcal{S}$ into two disjoint sets:
$\mathcal{N} \cup \mathcal{F} = \mathcal{S}$ where
$ \mathcal{N} \cap \mathcal{F} = \emptyset$. 

In a classification scenario, for example, two points belonging to the same class will be defined as {\em Near}; otherwise, they will be defined as {\em Far}.
Another example of an adjacency matrix is a neighborhood of a certain radius. 
For a distance metric $d_{ij} = d(\x_i,\x_j)$ in $\mathbb{R}^d$ and a given radius $r$ we define:
\begin{equation}
\begin{split}
\mathcal{N}=\left \{(i,j)~|~d_{ij}\leq r \right \} ~~\mbox{and}~~
\mathcal{F}=\left \{(i,j)~|~d_{ij}>r \right \}
\end{split}
\end{equation}

Denote the $p$-length binary code of point $\x_i$ by $\c_i \in \{\pm 1\}^p$. Our goal is to find $n$ binary codes $\{ \c_i \}_{i=1}^n$ that satisfy the following two requirements:

\begin{itemize}
\item {\em Compactness}: The length of the code should be short, i.e., $p$ should be as small as possible.
\item {\em Proximity Preserving}: The binary code should preserve the proximity of $\mathcal{X}$. That is, there exists a constant $\alpha$ s.t. $d_H(\c_i,\c_j) \leq \alpha$ for each pair $(i,j) \in \mathcal{N}$, and $d_H(\c_i,\c_j) > \alpha$ for each $(i,j) \in \mathcal{F}$, where $d_H(\cdot,\cdot)$ stands for the Hamming distance between two binary codes\footnote{In fact, this definition is twice the Hamming distance, but we stick with it for sake of clarity.}:
$$
d_H(\c_i,\c_j)=  \sum_{k=1}^p \left ( 1- \c_i[k] \c_j[k]
\right ) =  (p-\c_i^T \c_j).
$$
 \end{itemize}
It can be shown that if proximity relationships are determined according to $\ell_1$ or $\ell_2$ distance between
points in $\mathbb{R}^d$, the {\em Proximity Preserving} requirement can be fully satisfied using large enough codes (i.e., $p$ is large). However,
due to the { compactness} requirement, we want to relax the proximity preserving requirement and try to find an optimal code for a given code length. 

Denote a {\it proximity label}, $y_{ij}$, associated with each pair of points $(i,j) \in \mathcal{S}$:
$$
y_{ij} =
\left \{
\begin{array}{ll}
+1 ~~~\mbox{if}~~(i,j) \in \mathcal{N} \\
-1 ~~~\mbox{if}~~(i,j) \in \mathcal{F}
\end{array}
\right .
$$
For a given value $\alpha>0$ we define:
\begin{equation}
\label{eq:lij}
 z_{ij}= y_{ij} (\alpha-d_H(\c_i,\c_j))~.
\end{equation}
We would like that for each pair $(i,j),~z_{ij} \geq 0$, and accordingly we define a loss function:
\begin{equation}
\label{eq:total_loss}
l(z_{ij}) = \left \{
\begin{array}{ll}
1 & \mbox{if}~~ z_{ij} < 0 \\
0 & \mbox{otherwise}
\end{array}
\right .
\end{equation}
The empirical loss for the entire set reads:
\begin{equation}
\label{eq:lossL}
\mathcal{E}(\{y_{ij}\},\{\c_i\})= \min_{\alpha} \sum_{(i,j)\in \mathcal{S}} l(z_{ij})
\end{equation}
This loss penalizes pairs of points that are mislabeled, that is, pairs of points in $\mathcal{F}$ whose Hamming distance is smaller than $\alpha$, or pairs of points in $\mathcal{N}$ whose Hamming distance is larger than $\alpha$.

\begin{definition}[Proximity Preserving Code]
Given a set of data points $\mathcal{X}$ along with their proximity  labels, $\{y_{ij} \}$, a {\em Proximity Preserving Code} (PPC) of length $p$  is a binary code,
$\{ \c_i \}_{i=1}^n$, $\c_i \in \{\pm 1\}^p$, that minimizes
$\mathcal{E}(\{y_{ij}\},\{\c_i\})$.
\label{def:ppc}\end{definition}

 In the following we describe the procedure to generate the PPC. In particular, we show that finding PPC for a given set of points boils down to applying an integer low-rank matrix decomposition.
We provide two possible approximated solutions and show their connection to the minimum signed graph-cut problem. Finally, we provide a solution for extracting hashing functions for out-of-sample data points.

\section{Proximity Preserving Code}
\label{sect:ppc}
Recall the definition of $z_{ij}$  (Equation~\ref{eq:lij}): 
$z_{ij}= y_{ij} (\alpha-d_H(\c_i,\c_j))$.
Substituting the Hamming distance into this expression we get:
\begin{equation}
\label{eq:z_ij2}
z_{ij}= y_{ij}
\left(\alpha -
\left ( p-\c_i^T \c_j \right ) \right) = y_{ij} \left ( \c_i^T \c_j - \beta  \right )
\end{equation}
where we define $\beta=p-\alpha$.

To simplify notations we define a {\it code matrix} $C\in \{\pm 1\}^{p \times n}$ by
stacking the code words along its columns:
$$
C= \left (
\begin{matrix}
\mid & \mid &  & \mid \\
\c_1 & \c_2 & \cdots & \c_n \\ 
\mid & \mid &  & \mid 
\end{matrix}
\right )
$$
Similarly we define
$$
{B}=C^T C  ~~~~\mbox{where}~~ {B}_{ij} = \c_i^T \c_j
$$
Equation~\ref{eq:z_ij2} can now be defined over the entries of matrix ${B}$:
$$
z_{ij} = y_{ij}({B}_{ij}-\beta)
$$
and the total loss (Equation~\ref{eq:lossL}) is:
\begin{equation}
\label{eq:loss2}
\mathcal{E}(\{y_{ij}\},{B})=\min_{\beta} \sum_{ij} l(z_{ij})
\end{equation}
Denote by ${\b^k}$ the {\bf rows} of $C$ (similarly, the columns of $C^T$) such that 
$$
C^T = \left (
\begin{matrix}
\mid & \mid & & \mid  \\
\b^1 & \b^2 & \cdots & \b^p \\ 
\mid & \mid &  & \mid 
\end{matrix}
\right )
$$

Each  $\b^k \in \{\pm 1\}^n$ is a vector representing the $k^{th}$ bit of all the code words ($n$ words). The matrix ${B}=C^T C$ can now be represented as a linear sum
of $n \times n$ matrices:
\begin{equation}
{B} = \sum_{k=1}^p \b^k {\b^k}^T=\sum_{k=1}^p B^k
\label{eq:B}
\end{equation}
where $B^k = \b^k{\b^k}^T$ is a rank-1 matrix extracted from the $k^{th}$ bit of the code words. Thus, each additional bit can either increase the rank of matrix $B$ or leave it the same. 
Our goal then is to
find a low rank matrix ${B}=C^T C$,  minimizing the
loss defined in Equation~\ref{eq:loss2}.

The minimization function defined in Equation \ref{eq:loss2} introduces a combinatorial problem which is NP-hard. Therefore we relax the binary loss function and re-define it using a logistic loss function:
$$
l(z_{ij})=\tilde \ell(y_{ij}({B}_{ij}-\beta) )
$$
where $\tilde \ell(z)=\ln(1+e^{-z})$.
The relaxed total loss is therefore
\begin{equation}
\label{eq:loss3}
\mathcal{E}(\{y_{ij} \},{B})= \min_{\beta} \sum_{ij} \tilde \ell(y_{ij}({B}_{ij}-\beta) )
\end{equation}


\subsection{Bit Optimization}

In the proposed process we generate the codes for $n$ data points in a sequential manner, bit after bit. In the following we detail the minimization process for bit $k$. This is also illustrated in Figure \ref{fig:algo}.

At this step we assume that $k-1$ bits of PPC code have already been generated. 
Denote 
$${B}^{1:k}=\sum_{\ell=1}^{k} B^\ell~~\mbox{where}~~B^\ell = \b^\ell {\b^\ell}^T.
$$
For the $k^{th}$ bit, we minimize  Equation~\ref{eq:loss3} with respect to ${B}^k$ and $\beta$ as follows:
\begin{equation}
\label{eq:loss4}
\mathcal{E}^k= \sum_{ij} \tilde \ell(y_{ij}({B}^{1:k-1}_{ij} + {B}^k_{ij}-\beta) )
\end{equation}
Note that $B_{ij}^{1:k-1}$ is already known at step $k$. As mentioned above, $\mathcal{E}^k$ is minimized using alternate minimization, described below. \\

\noindent
\underline{Step I - optimizing $\beta$}:\\
$\mathcal{E}^k$ is convex with respect to $\beta$, so any scalar search is applicable here. 
Since the loss $\tilde \ell(z)$ is nearly linear for $z\le0$, a fast yet sufficiently accurate
approximation for $\beta$ is to choose the
value that equates the number of misclassified pairs in the $\mathcal{N}$ and $\mathcal{F}$ sets.
For the current code $\{\c_i\}_{i=1}^n,~\c_i \in \{\pm 1\}^{k-1}$, and a constant value $\alpha$, define the misclassified sets:
$$
E_N(\alpha) = \left \{ (i,j) ~|~ (i,j) \in \mathcal{N}\ ~\mbox{and}~
d_H(\c_i,\c_j)>\alpha \right \}
$$
and similarly
$$
E_F(\alpha) = \left \{ (i,j) ~|~ (i,j) \in F ~\mbox{and}~ d_H(\c_i,\c_j) \le
\alpha \right \} 
$$
The value of $\alpha$ is set such that the cardinality of the two sets is equal, i.e., the $\hat \alpha$ that satisfies:
\begin{equation}
|E_N(\hat \alpha)|=|E_F(\hat \alpha)|
\label{eq:beta-eq}\end{equation}
and accordingly $\hat \beta =  (t-1) - \hat \alpha$.
This is visualized in Step I of Figure \ref{fig:algo}. We show a histogram of the near pairs of samples in blue and the far pairs in red, and $\alpha$ is the vertical black line thresholding the Hamming distance.
\\

\noindent
\underline{Step II - optimizing $\b^k$}: \\
For the evaluated $\hat \beta$,
Equation~\ref{eq:loss4} becomes:
\begin{gather*}
\mathcal{E}^k = \sum_{ij} \tilde \ell(y_{ij}({B}^{1:k-1}_{ij} + B^k_{ij}-\hat \beta) )\\
= \sum_{ij} \tilde \ell(y_{ij}(B^k_{ij}+\gamma^{k-1}_{ij}) )
\end{gather*}
where we define $\gamma^{k-1}_{ij} = {B}^{1:k-1}_{ij}- \hat \beta$. In a forward
greedy selection process, we approximate the potential decrease in the loss
using the gradient. Our goal is to find $B^{k}$ that minimizes $\mathcal{E}^k \approx \mathcal{E}^{k-1} + \Delta \mathcal{E}^k$ or alternatively maximizes $-\Delta \mathcal{E}^k$ where:
$$
-\Delta \mathcal{E}^k=-\sum_{ij} \frac{\partial \mathcal{E}^{k}}{\partial B^{k}_{ij}}B^k_{ij}
= -\sum_{ij}  \tilde \ell'(y_{ij} \gamma_{ij}^{k-1}) y_{ij} B^k_{ij}
$$
where $\tilde \ell'$ stands for the derivative of the logistic loss function
$\tilde {\ell}'(z)=-1/(1+e^z)$.

Defining $w_{ij}=-y_{ij} \tilde \ell'(y_{ij}\gamma^{k-1}_{ij})$, we arrive at the following maximization problem:
$$
\max_{B^k} \sum_{ij}  w_{ij} B^k_{ij}  = \max_{\b^k} \sum_{ij}  w_{ij} \b^k[i] \b^k[j]
$$
where the maximization is taken over all entries of $\b^k \in \{\pm 1\}^n$.
For the sake of simplicity we omit the superscript $k$ and denote $\b^k$ by $\b$.  Collecting $\{w_{ij}\}$ into matrix $W$, s.t. $W(i,j)=w_{ij}$, the above maximization can be simply expressed in a matrix form:
\begin{equation}
\label{eq:bWb}
\hat \b = \argmax_{\b} \b^T  W \b~~~~~\mbox{s.t.}~~\b \in \{\pm 1\}^n
\end{equation}


If the weight matrix W was all positive (all entries are positive values), this problem can be interpreted as a graph min-cut problem.
 In our problem, however, the matrix W is comprised of both positive and negative values, indicating
pairs $(i,j)$ that are properly and improperly assigned as near or far according to the code computed. This is termed in the literature a {\em signed min-cut} problem which is equivalent to the {\em max-cut} problem whose solution is NP-hard.

In the proposed solution we start with an initial guess for the bit vector $\b$ and improve it by using a forward greedy selection scheme. We present two iterative approaches for the selection scheme: {\em vector update} and {\em bit update}.

\subsubsection{Vector Update}
Given an initial guess for $\b$, the {\em vector update} method updates the entire vector at once. At each iteration the vector is improved by applying: 
\[{\b'}=sign(W \b)\] 
where $\b'$ is the updated vector that satisfies: $\b'^T W \b' \geq  \b^T W  \b$. The following four theorems prove that $\b'$ is a better vector than $\b$:

\begin{theorem}\label{theo:bwb}
For any $n \times n$ matrix $W$ and $\b,\b' \in \{\pm 1\}^n$, ~if ~$\b'=sign(W \b)$,  then $\b'^T W \b \geq \b^T W \b$
\end{theorem}

\begin{theorem}
\label{alternate}
Assuming  $ W $  is Positive Semidefinite, if 
$ {\b'}^TW \b > \b^TW \b $, then 
$ {\b'}^TW{\b'}>\b^TW\b $.
\end{theorem}

\begin{theorem}\label{theo:making-PD}
Any symmetric matrix $W$ can become Positive Semidefinite by applying $W\leftarrow W+ |\lambda| I$ where $\lambda$ is the smallest eigenvalue of $W$.
\end{theorem}

\begin{theorem}\label{PD}
Adding a constant value to the diagonal of the weight matrix W will not affect the output code computed.
\end{theorem}
Proofs are given in the Appendix.
Using the above theorems, Algorithm~ \ref{alg1} summarizes the {\em vector update} iterations.

\begin{algorithm}[h]
\caption{Vector Update (${\b},W$)}
\label{alg1}
\begin{algorithmic} 
\State $ \hat \b \gets {\b}$ 
\State $W \leftarrow W+|\lambda| I$ where $\lambda$ is the smallest e.v. of $W$
\Repeat 
	\State $ \b \gets {\hat \b}$ 
	\State ${\hat \b} \gets sign(W\b)$
\Until{ ${\hat \b} = \b $}
\State
\Return {$\b$}
\end{algorithmic}
\end{algorithm}

\subsubsection{Bit Update}
Unlike the {\em vector update}, the {\em bit update} method changes one bit at a time: For each bit in vector $\b$, we flip the bit and determine whether the new value improved the objective $\b^T W  \b$.
This is repeated for each bit sequentially, and over the entire vector, until convergence. 
This procedure can be applied very efficiently using the following scheme:  Define $ \b=\b_{(i)}+\b_{(-i)}$, where $\b_{(i)}=(0... b_{i}...0)$ is a one-hot vector with  the $i^{th}$ entry of $\b$ at the $i^{th}$ coordinate. Accordingly, $\b_{(-i)}=\b -  \b_{(i)}$ is   the vector $\b$ with $0$ at the $i^{th}$ coordinate. When optimizing the $i^{th}$ bit:
\begin{eqnarray*}
\b^T W \b = (\b_{(i)}+{\b_{(-i)})}^TW( \b_{(i)}+ \b_{(-i)})= \\
= \b_{(i)}^2 W(i,i)+2{ \b^T_{(-i)}} W \b_{(i)}+const
\end{eqnarray*}
It can be verified that the only term affecting the optimization is \({\b^T_{(-i)}} W \b_{(i)}\). Therefore we can optimize each bit in $\b$ by looking at the value of the $i^{th}$ element of ${\b^T_{(-i)}} W$. Subsequently, the only elements affecting this value in the matrix $W$ are in the $i^{th}$ column of $W$. Thus,
 \begin{equation}
 \label{eq:fgss}
\b[i]' =  sign\left ( { \b^T_{(-i)}} W[:,i] \right)
\end{equation}
where we denote by $W[:,i]$ the $i^{th}$ column of W. 
We apply this optimization scheme for each bit sequentially, and repeatedly over the entire vector $\b$, until convergence. Each bit update is inserted immediately into $\b$ so that the optimization for the $i+1$ bit will account for the preceding bits that have been calculated. 
This update is computationally inexpensive, requiring $O(n)$ operations for each bit update, and $O(n^2)$ operations for one round over the entire $\b$. 
This method is summarized in Algorithm~\ref{alg2}.

\begin{algorithm}
\caption{Bit Update ($\b,W$)}
\label{alg2}
\begin{algorithmic} 
\Repeat 
	\State $\hat{\b} \gets \b$
	\For {$i\gets 1, N$}
		\State $ \b_{-i} \gets \b, ~~~~  \b_{-i}[i] \gets 0$
		\State $\b[i] \gets sign(\b_{-i}^T W[:,i])$
	\EndFor
\Until{ $\hat{\b} = \b $}
\State
\Return {$\b$}
\end{algorithmic}
\end{algorithm}

The two algorithms presented for the iterative bit optimization scheme provide a solution to the max-cut problem where both positive and negative weights appear on the graph edges. The iterations require several light computations and stop when a local maximum is reached and the iteration scheme can no longer improve upon the current bit vector. We show in the Experiments Section that the bit update scheme achieves better codes than the vector update and is therefore preferable.

\subsection{Initial Guess}
Our method is based on an iterative scheme. Therefore, we start the optimization with an initial guess and improve upon it.
A common solution is to relax the constraints $\b[i] \in \pm 1$ and allow real-valued solutions. This enables the maximization problem to be cast as an eigenvalue problem. The final solution is then obtained by thresholding the results, in a similar manner to \citeauthor{sh} \shortcite{sh}.  Interestingly, we have found that starting from a {\em random guess} and applying the suggested iterations produces {more accurate solutions} and with {much faster} compute time than the traditional eigenvalue solutions.

In conclusion, the algorithm provided above can be used to solve the signed graph min-cut problem where $W$ consists of both positive and negative weights. We show empirically that our solution is equivalent to or outperforms other methods, by starting from a random guess solution and applying an update scheme until convergence. Note, that the proposed update schemes can improve upon any approximated solution suggested in the literature, as the suggested iterations do not deteriorate and can only improve the objective function. Our evaluations and experimental results are provided in the Experiments Section.

\section{Signed Graph Min-Cut Problem}
\label{sect:signedCut}
Equation \ref{eq:bWb} suggests that our problem can be cast as a signed graph min-cut problem. A weighted graph is represented by a vertex set $V=\{1,\cdots,n\}$ and weights $W_{ij}=W[i,j]=W[j,i]$ for each pair of its vertices $(i,j)\in V \times V$. 
The weight of the minimum cut $w(G,\bar G)$ is given by the following problem:
\begin{equation}
\label{eq:int_prog}
	Minimize ~~~~\frac{1}{2} \sum_{i,j} W_{ij} (1-b_i b_j)~~s.t.~~b_i\in \{\pm 1\}.
\end{equation}
where $\b=[b_1,\cdots,b_n]$ is an indicator vector s.t. $b_i=1$ if $i\in G$ and $b_i=-1$ if $i\in \bar G$. The above minimization is an integer quadratic program whose solution is known to be NP-hard \cite{alon2004approximating}. Note that the above formulation can be expressed similarly by $maximize~~~\b^t W \b~~s.t.~~\b \in \{\pm 1\}^n$, which is similar to the expression given in Equation \ref{eq:bWb}.
The weights collected in Equation~\ref{eq:int_prog} refer only to pairs $(i,j)~s.t.~~b_i \neq b_j$. Thus, the minimal cut aims at including as many negative weights as possible while excluding positive weights. Since we are dealing with signed graphs, balancing the cut is not critical as it is in unsigned graphs since cutting a small component with few edges does not necessarily provide the smallest cut. 

\citeauthor{alon2004approximating} \shortcite{alon2004approximating} define the above problem as a $\|W\|_{\infty \rightarrow 1}$ norm and provide a semidefinite relaxation: $ maximize  \sum_{ij} W_{ij} {\bf u}_i \cdot {\bf v}_j$ ~s.t. $\|{\bf u}_i\|=\|{\bf v}_j\|=1$. The semidefinite program can be solved within an additive error of $\epsilon$ in polynomial time. Alon and Naor suggested three techniques to round the semidefinite solution into a binary solution ($b_i \in \{\pm 1\}$), which provides an approximation to the original solution up to a constant factor ($K_G$, called {\em Grothendieck's constant},  $1.570 \leq K_G \leq 1.782$). In the Experiments Section we show that our iterative update approach can improve over Alon and Naor's solution when taking their solution as an initial guess. Moreover, taking a random guess as an initial solution provides a final solution that is comparable or better, so the benefit of using a costly approximated solution as initial guess is questionable. 

The minimization in \ref{eq:int_prog} can be equivalently rewritten in a quadratic form:
\begin{equation}
\label{eq:laplacian}
	Minimize ~~~~\frac{1}{2} \sum_{i,j} W_{ij} (b_i -b_j)^2~~s.t.~~b_i\in \{\pm 1\}
\end{equation}
The matrix form of the above minimization reads:
$mimimize~~~\b^t L \b~~s.t.~~\b \in \{\pm 1\}^n$, where $L=D-W$ is the Laplacian of $W$ and  $D$ is a diagonal matrix $D_{ii}=\sum_j W_{ij}$. 
The Laplacian of a graph is frequently used for graph clustering or graph-cut using spectral methods. It was shown that taking the second-smallest eigenvector (the Fiedler vector) and thresholding it at zero provides a relaxed approximation for the minimization in Eq.~\ref{eq:laplacian} (see \citeauthor{von2007tutorial} for more details). However, spectral methods commonly deal with positive weights where the matrix $W$ is guaranteed to be positive semidefinite. This is not the situation in our case where eigenvalues might be negative as well.

\citeauthor{kunegis2009slashdot} \shortcite{kunegis2009slashdot} suggested an alternative for graph Laplacian for signed-graphs: $\bar{L} = \bar{D} - W$, where $\bar{D}_{ii} = \sum_j{\mid W_{ij} \mid }$ and proved that $\bar L$ is positive semidefinite. However, \citeauthor{knyazev2017signed} \shortcite{knyazev2017signed} argue that the signed Laplacian does not give better clustering results than the original definition of Laplacian, even if the graph is signed.  We show in our experiments that neither solution works as well as the greedy update scheme suggested in this paper.

\subsection{Optimizing the Hashing Functions}
\label{sec:hash}
\label{sect:oos}
Finally, we arrive at the out-of-sample extension and explain how to learn hashing functions to encode out-of-sample data points. We found that it is preferable to first optimize for the binary vector $\b^k$, then learn a hashing function $h^k(\x)$ requiring   $h^k(\x_i)=\b^k[i]$, for $i=1..n$.
 Optimizing directly for the hashing function yields a non-linear optimization that often provides inaccurate results. Splitting the optimization into two steps allows each step to be exploited in the best manner. We assume that novel data points will be drawn from the same distribution of the given data $\mathcal{X}$. Therefore, the hashing functions can be optimized using the empirical loss over $\mathcal{X}$. 

We denote by $\tilde \b^k$ the optimal $\b^k$ resulting from the first step. This vector encodes the optimal binary values for the $k^{th}$ bit (over all data points). We then train a binary classifier $h^k(\x ; \Theta)$ over the input pairs $\{(\x_i, \tilde \b^k[i])\}$, by minimizing a loss function:
$$
\min_{\Theta} \sum_{i=1}^n \mathcal{L}(h^k(\x_i,\Theta),\tilde \b^k[i] )
$$
 where $\Theta$ denotes the classifier's parameters. We use kernel SVM \cite{svm} with Gaussian kernels to classify the points $\{ \x_i\}$  into \(\pm1\), but any standard classifier can be applied similarly. 
 At step $k$, we train the hash functions $h^k$ and construct the $k^{th}$ bit for the binary codes: ${\b^k} = [h^k(\x_1),h^k(\x_2),\cdots, h^k(\x_n)]^T$. This bit is updated immediately in the codes $\{ \c_i \}_{i=1}^n$, allowing the $k+1$ bit to account for the errors in $\b^k$.  This  error correcting scheme is another benefit of the two-step solution.

As we proceed, the algorithm adds more bits to the PPC code. Each additional bit is aimed at decreasing the total loss. The process terminates when the total loss is below a given threshold, or when the number of bits exceeds $p$.

\section{Experiments and Results}
\label{sect:exp}

\begin{figure}[tbh]
\includegraphics[width=\columnwidth]{{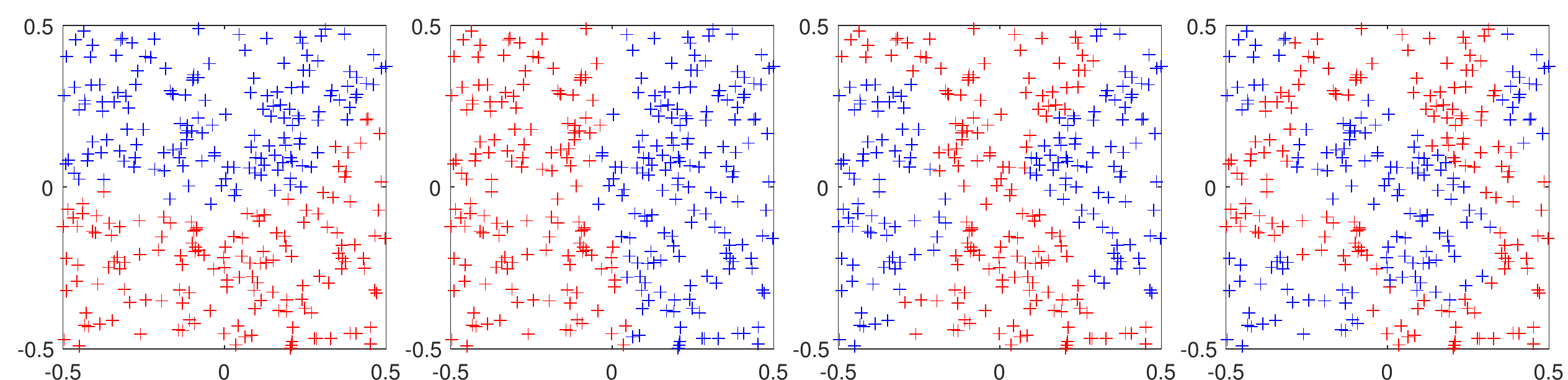}}
\caption[Visualization of first 4 bits of PPC code]{Visualizations of the bits assigned to each data point in a 2-dimensional space according to PPC. In these images, blue points have been assigned 1 while red points have been assigned -1.}
\label{figure:first4bits}
\end{figure}

\begin{figure}[!tb]
    \centering
    \includegraphics[width=0.32\columnwidth]{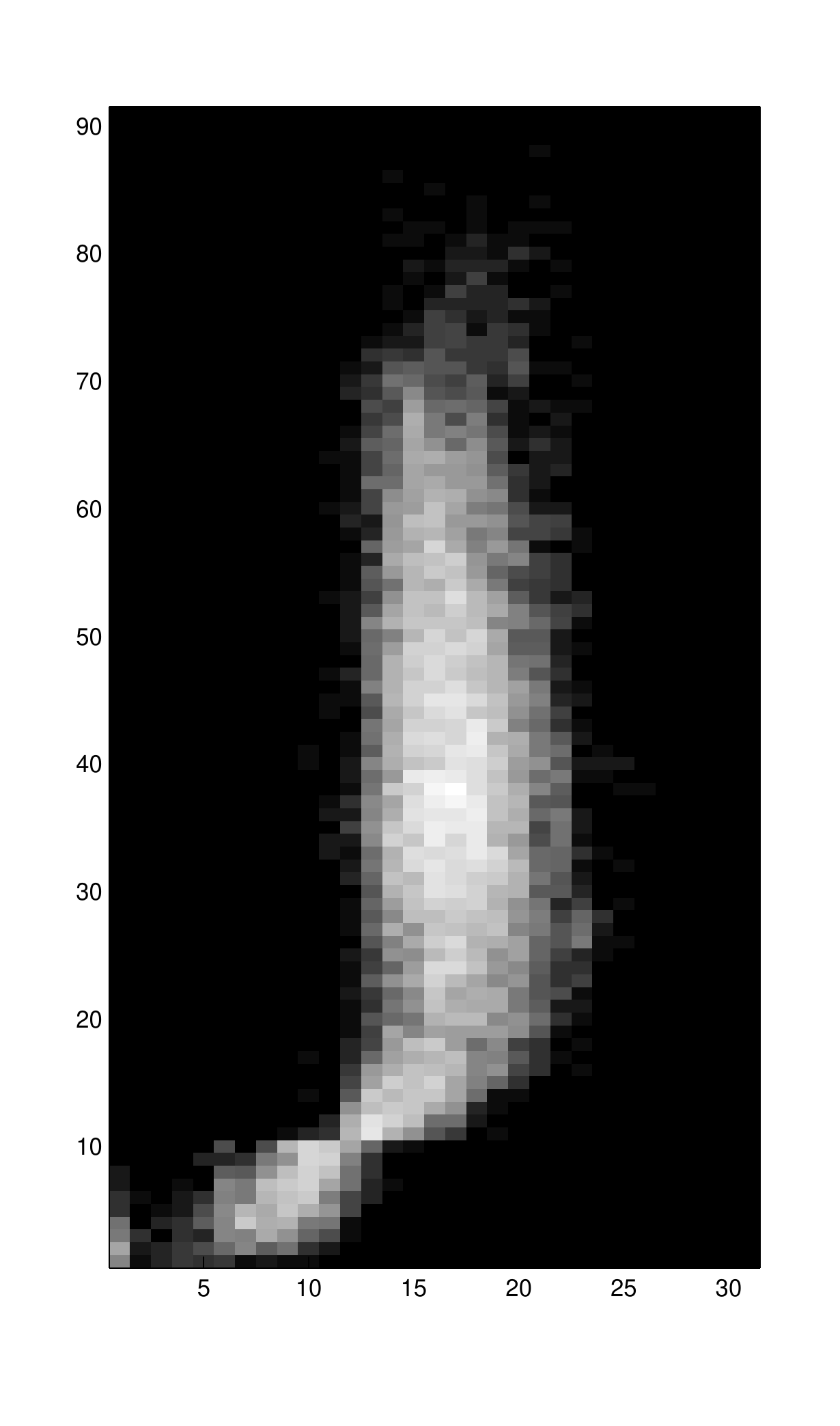}
    \includegraphics[width=0.32\columnwidth]{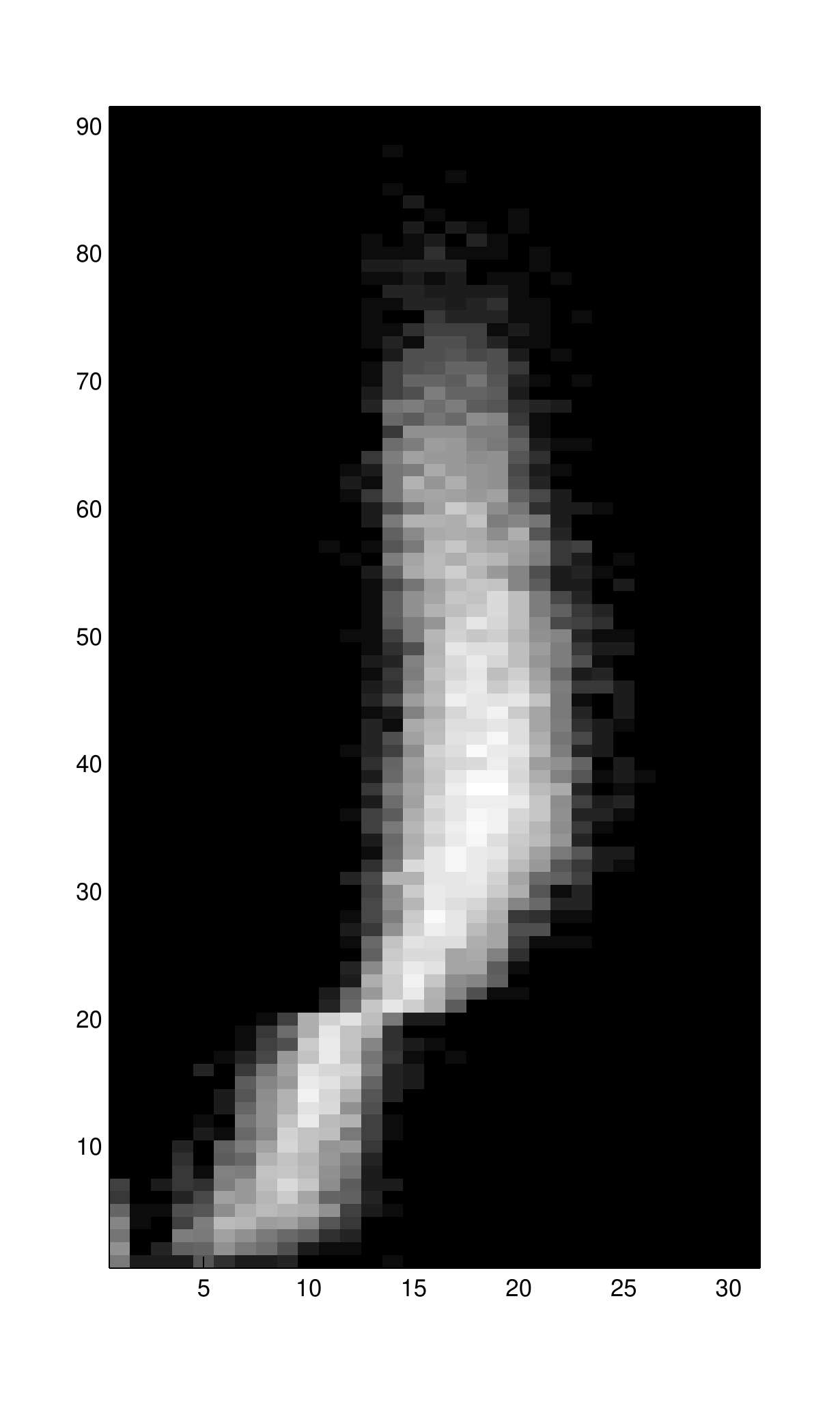}
    \includegraphics[width=0.32\columnwidth]{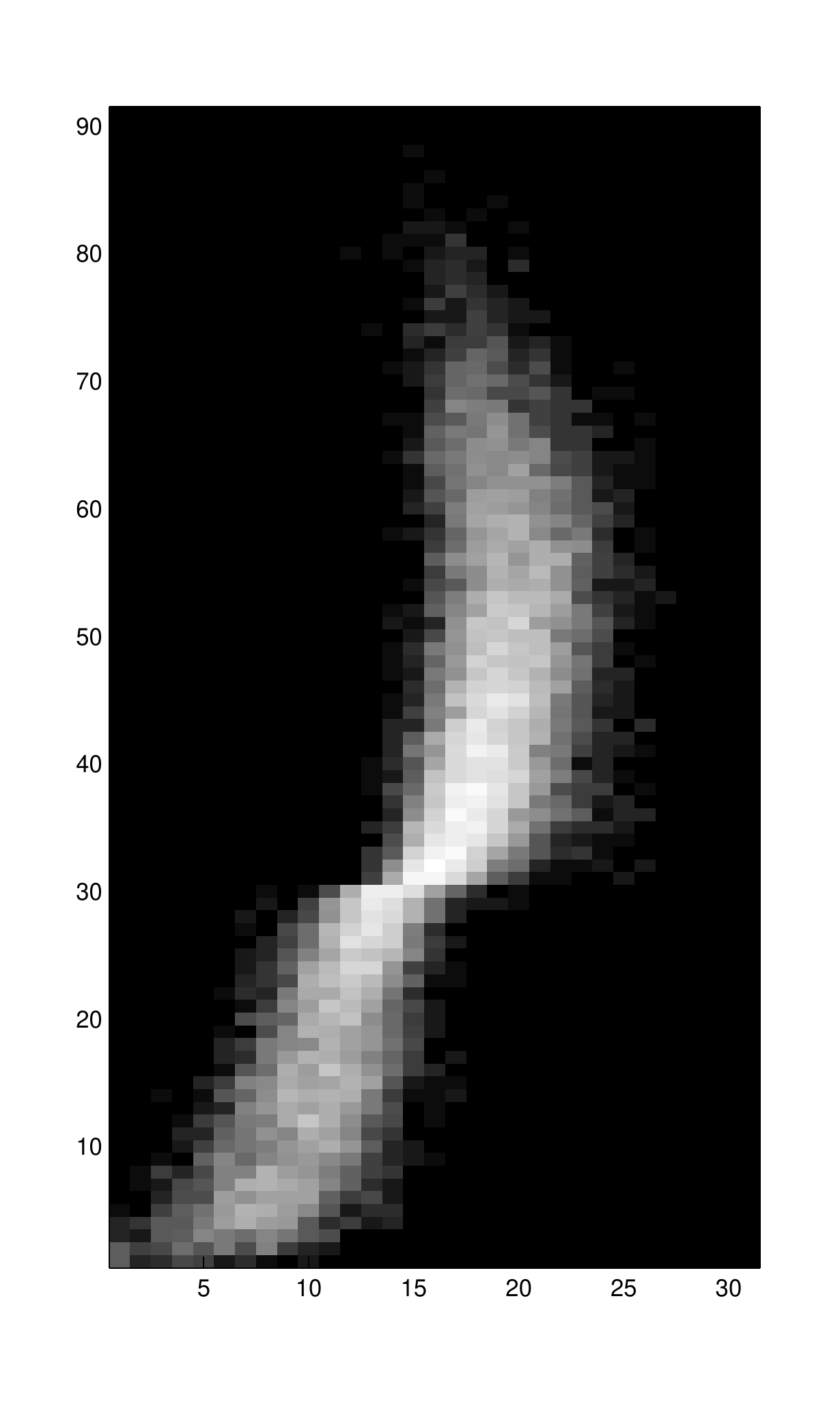}
    \caption[Joint histogram of Hamming distances and Euclidean distances for PPC code]{Joint histogram of Hamming distances (x-axis) with respect to the Euclidean distances (y-axis) for PPC code computed on data with r-neighborhood affinity of (left-to-right) $r=10$, $r=20$ and $r=30$.}
    \label{fig:jhist}
\end{figure}

\begin{figure*}[!tbh]
    \begin{subfigure}[b]{0.32\textwidth}
         \includegraphics[width=\linewidth]{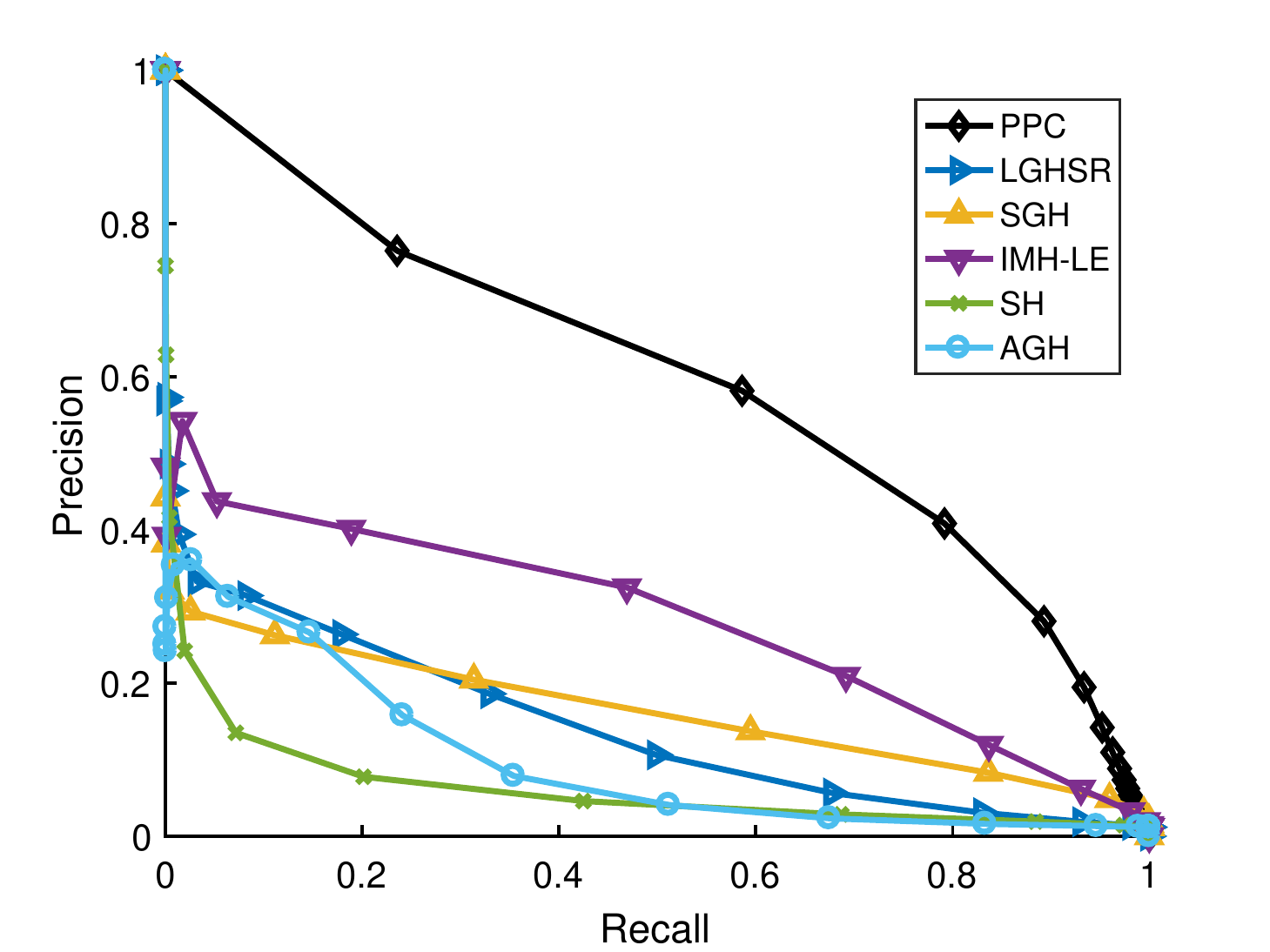}
        \caption{LabelMe (self-supervised) \label{fig:labelme-prec-rec}}
    \end{subfigure}%
    \begin{subfigure}[b]{0.32\textwidth}
         \includegraphics[width=\linewidth]{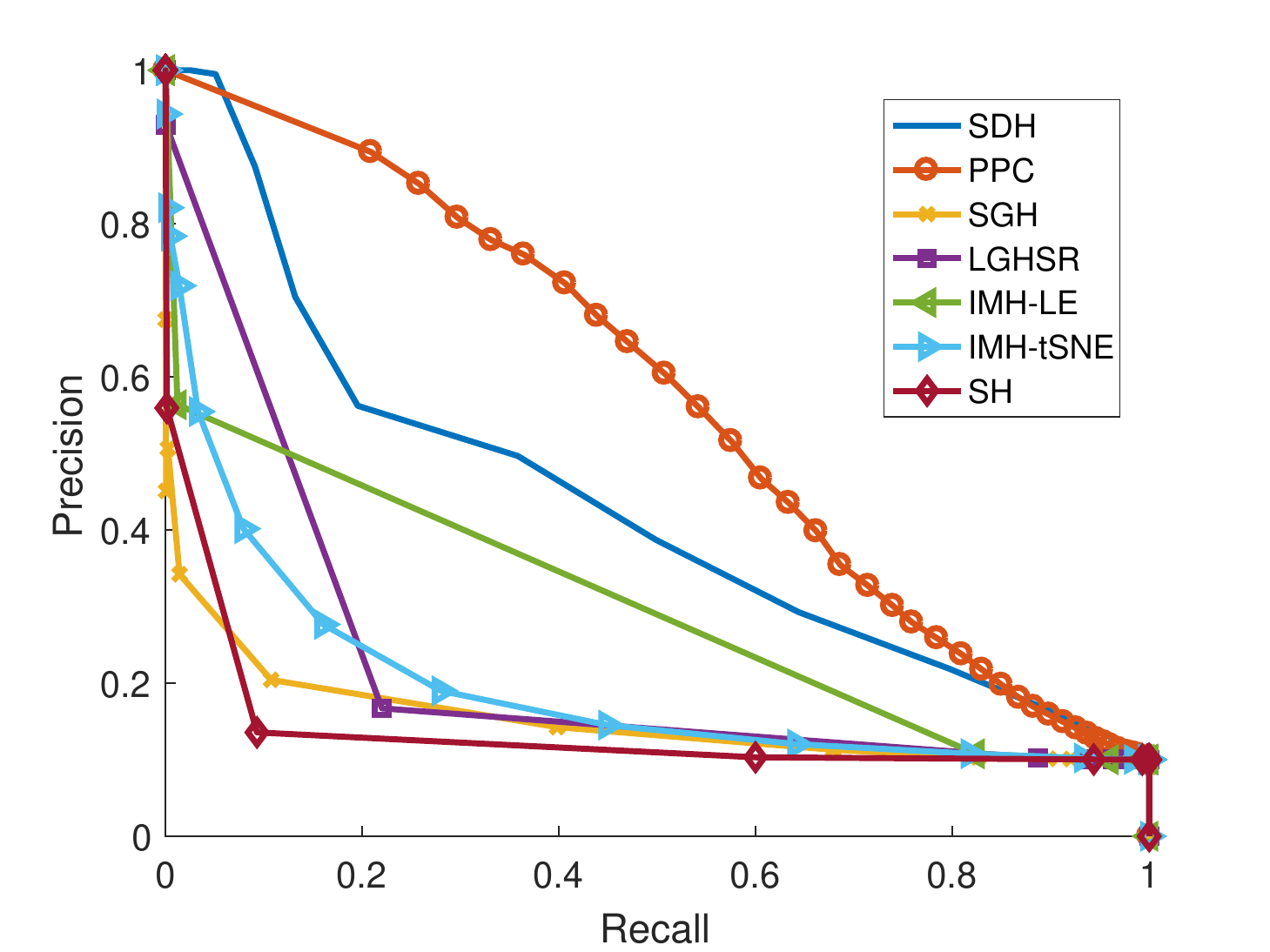}
        \caption{MNIST\label{fig:mnist-prec-rec}}
    \end{subfigure}%
    \begin{subfigure}[b]{0.32\textwidth}
         \includegraphics[width=\linewidth]{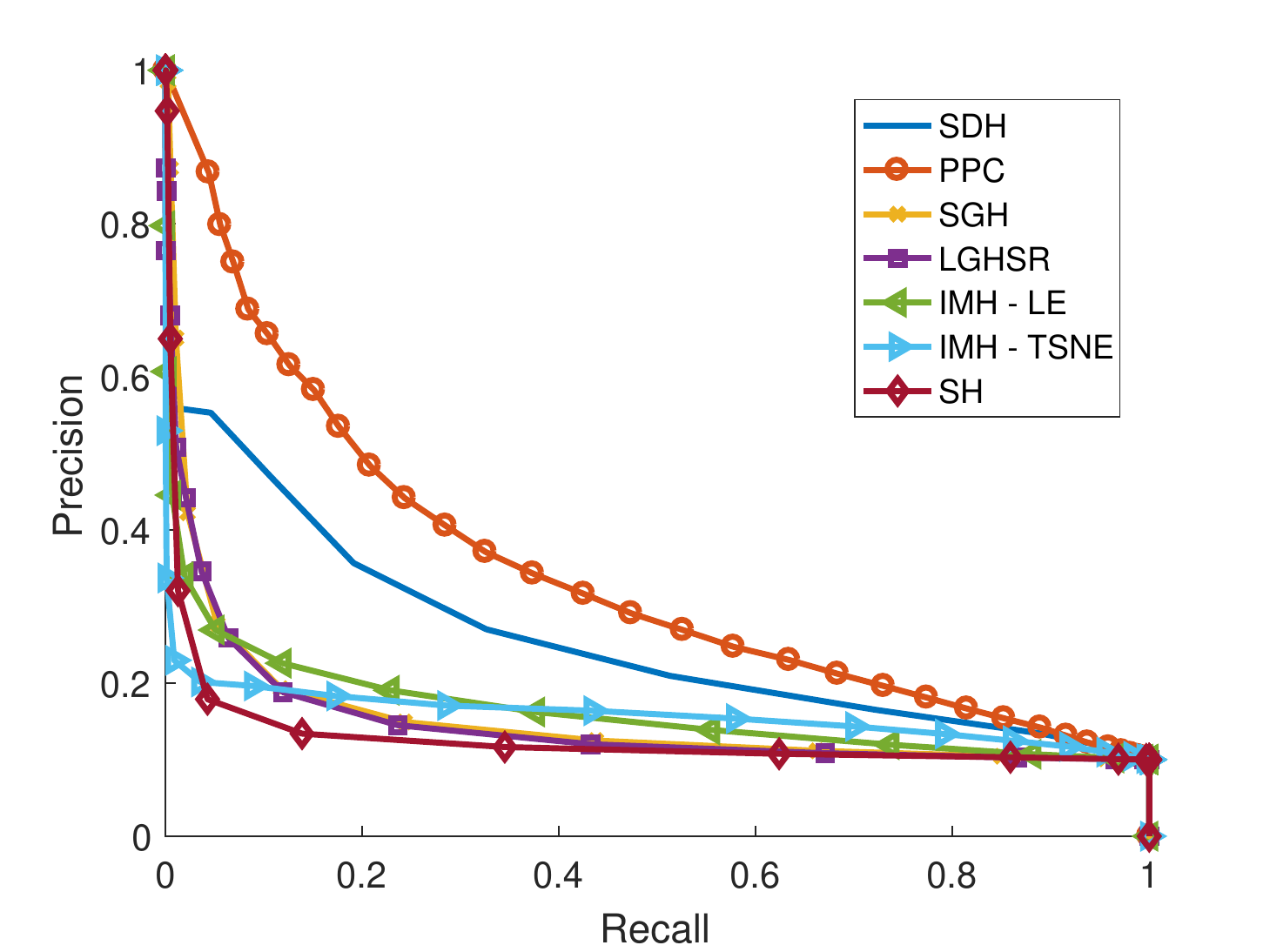}
        \caption{CIFAR-10\label{fig:cifar-prec-rec}}
    \end{subfigure}%

    \caption{Precision-recall of 50 bit codes with varying Hamming threshold $\alpha$ for different datasets. The methods compared: Spectral Hashing (SH) \cite{sh}, Anchor Graph Hashing (AGH) \cite{agh}, Inductive Manifold Hashing (IMH) \cite{imh}, Scalable Graph Hashing (SGH) \cite{sgh}, Supervised Discrete Hashing (SDH) \cite{sdh}, and Large Graph Hashing with Spectral Rotation (LGHSR) \cite{lghsr}.}
    \label{fig:prec-rec}
\end{figure*}

\begin{table*}[!tbh]
\centering
{%
\begin{tabular}{c|cccccccc}
\hline
\textbf{\begin{tabular}[c]{@{}c@{}}AUC\end{tabular}} & \multicolumn{8}{c}{\textbf{CIFAR-10}} \\ \hline
\textbf{Code Length} & \textbf{12} & \textbf{16} & \textbf{24} & \textbf{32} & \textbf{48} & \textbf{64} & \textbf{96} & \textbf{128} \\ \hline
SH & 0.121379 & 0.121498 & 0.119628 & 0.121108 & 0.126965 & 0.12771 & 0.130133 & 0.129777 \\
IMH-tSNE & 0.154169 & 0.165783 & 0.168781 & 0.150094 & 0.165669 & 0.16049 & 0.163685 & 0.174634 \\
IMH-LE & 0.170701 & 0.164687 & 0.144931 & 0.154949 & 0.165396 & 0.152761 & 0.162143 & 0.152876 \\
SGH & 0.129056 & 0.12776 & 0.131998 & 0.138006 & 0.136264 & 0.144698 & 0.153421 & 0.157661 \\
LGHSR & 0.136739 & 0.144933 & 0.149855 & 0.148962 & 0.144178 & 0.144102 & 0.150296 & 0.147022 \\
SDH & 0.249316 & 0.191575 & 0.229962 & 0.250167 & 0.227537 & 0.257303 & 0.288238 & 0.326233 \\
PPC & \textbf{0.28365} & \textbf{0.312302} & \textbf{0.308905} & \textbf{0.329332} & \textbf{0.343186} & \textbf{0.352296} & \textbf{0.354291} & \textbf{0.355555} \\ \hline
\end{tabular}%
}
\caption{Area under the curve of precision-recall for varying code lengths on the CIFAR-10 dataset. The methods compared: Spectral Hashing (SH) \cite{sh}, Inductive Manifold Hashing (IMH) \cite{imh}, Scalable Graph Hashing (SGH) \cite{sgh}, Large Graph Hashing with Spectral Rotation (LGHSR) \cite{lghsr}, Supervised Discrete Hashing (SDH) \cite{sdh}, and our method (PPC).}
\label{table:cifar-test-test}
\end{table*}

\begin{figure*}[!tbh]
\includegraphics[width=\textwidth]{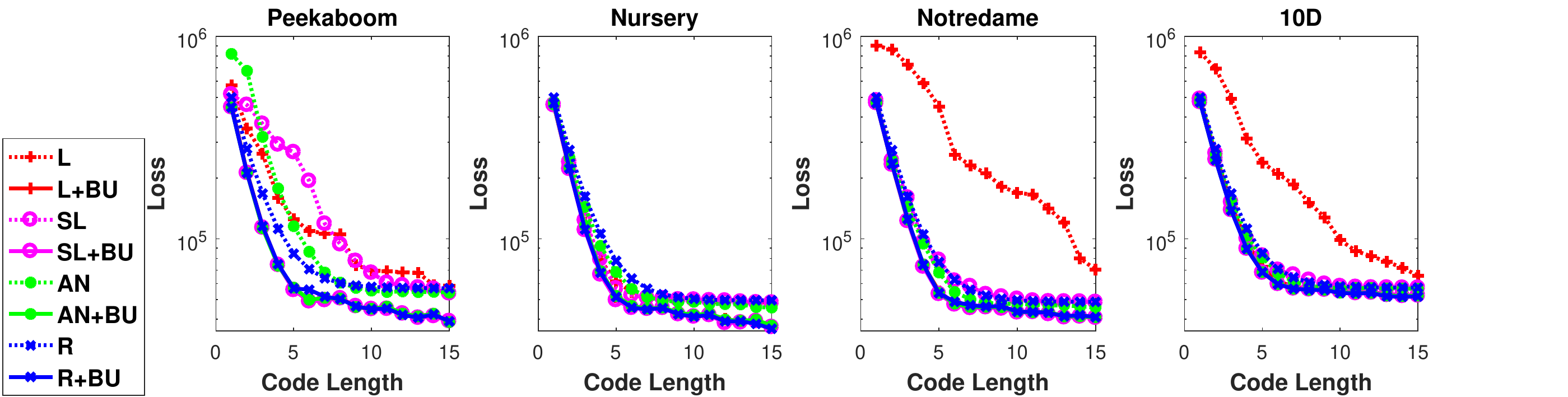}\caption{Loss of binary code as a function of code length for PPC with different initial guesses and iterative improvements. This is shown for 4 small datasets. The initial guesses are: signed eigenvector corresponding to the smallest non-trivial eigenvalue of the Laplacian (L) as described in \citeauthor{von2007tutorial}, signed Laplacian (SL) as suggested by \citeauthor{kunegis2009slashdot}, the sign of the random projection of the 3 smallest non-trivial eigenvalues of the original definition of Laplacian as suggested by \citeauthor{alon2004approximating} (AN), and random guess (R). We show here results for PPC using only the initial guess, and bit update (BU). \label{fig:insample-loss}}
\end{figure*}

\begin{figure*}[htb]
\includegraphics[width=\textwidth]{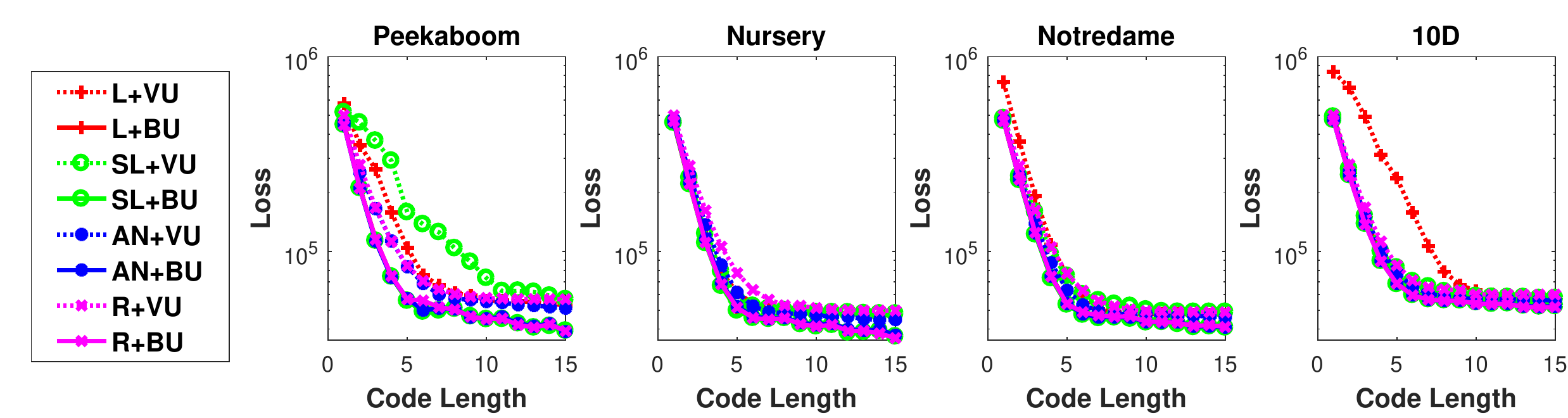}\caption{Loss of binary code as a function of code length for PPC with different initial guesses and iterative improvements. This is shown for 4 small datasets. Here we compare the effects of bit update (BU), and vector update (VU). The initial guesses are: signed eigenvector corresponding to the smallest non-trivial eigenvalue of the Laplacian (L) as described in \citeauthor{von2007tutorial}, signed Laplacian (SL) as suggested by \citeauthor{kunegis2009slashdot}, the sign of the random projection of the 3 smallest non-trivial eigenvalues of the original definition of Laplacian as suggested by \citeauthor{alon2004approximating} (AN), and random guess (R). \label{fig:insample-loss-vubu}}
\end{figure*}

To illustrate the optimization process, a synthetic example is shown in Figure~\ref{figure:first4bits}. In this figure 300 points are drawn in 2D in a range of $[-0.5..0.5]\times[-0.5..0.5]$. The figure shows the first 4 first bits (red/blue indicate $\pm 1$) where the proximity matrix was generated with a r-neighborhood proximity measure. This demonstrates how the PPC algorithm tries to separate the vector space into two labels by balancing between areas with high neighbor density and correcting for the errors of previous bits.

The mutual relationships between the actual distances and the Hamming distances are illustrated in Figure~\ref{fig:jhist}, which shows the joint histogram of $d_H(\c_i,\c_j)$ vs. the Euclidean distances, $d_{ij}$ for three cases. This is another synthetic 2D example with varying r-neighborhood proximity measures. The x-axis indicates the Hamming distances of the generated code and the y-axis indicates the actual Euclidean distances. The histograms are plotted as gray-scale images where the gray-value in each entry indicates the number of pairs with the associated distances.
The brighter the gray-value, the greater the number of pairs (we display the log of the actual values for a better visualization). For each case we see that most of the pairs with Euclidean distance below $r=10,20,30$ (the pairs labeled as "Near" in this example) are concentrated to the left of the respective Hamming distances, $d_H=11,12,13$. These were also the final $\alpha$ values at the last step of each case. 
It is interesting to note that the conditional distributions of $d_H$ at the two sides of the alpha values are wide, while the order of the Euclidean distances is not necessarily preserved in the respective Hamming distances. This indicates that the bits are allocated solely to optimize the neighborhood constraints, and not to meet any other requirements such as preserving the ordinal distances. This allows for optimal allocation of the bit resources. 

We evaluate Proximity Preserving Code on several public datasets: MNIST \cite{mnist}, CIFAR-10 \cite{cifar}, and LabelMe \cite{labelme}. 
CIFAR-10 \cite{cifar} is a labeled subset of the 80 million tiny images dataset, consisting of 60,000 32x32 color images represented by 512-dimensional GIST feature vectors \cite{GIST}. It is split into 59,000 images in the training set and 1000 in the test set.
MNIST \cite{mnist} is the well-known database of handwritten digits in grayscale images of size 28x28. The dataset is split into a training set of 69,000 samples and a test set of 1,000 samples. 
LabelMe \cite{labelme} has 20,019 training images and 2000 test images, each with a 512D GIST descriptor. The descriptors were dimensionality reduced to 40D using PCA. We use this dataset as unsupervised, and the affinity is defined by thresholding in the Euclidean GIST space such that each training point has an average of 100 neighbors.

We evaluate the results by computing a precision-recall graph of varying Hamming thresholds (denoted by $\alpha$ in Equation \ref{eq:lij}). We compare our methods to the following state-of-the-art spectral hashing methods: Spectral Hashing (SH) \cite{sh}, Anchor Graph Hashing (AGH) \cite{agh}, Inductive Manifold Hashing (IMH) \cite{imh}, Scalable Graph Hashing (SGH) \cite{sgh}, Supervised Discrete Hashing (SDH) \cite{sdh}, and Large Graph Hashing with Spectral Rotation (LGHSR) \cite{lghsr}. We use the default settings that the authors provided, and as in \citeauthor{imh} \shortcite{imh} we use settings of anchor number $m=300$ and neighborhood number $s=3$. For IMH, we show results using both Laplacian eigenmaps (LE) and t-SNE. 

We first compare our method in the unsupervised (or self-supervised) setting to the unsupervised methods listed above. Results are shown in Figure \ref{fig:labelme-prec-rec}. We show the precision-recall graph of the LabelMe dataset with the self-supervised affinity labels. The results are for the 50 bit code computed for the train set vs. the test set. The results clearly show that our code is more accurate than the other methods over {\em all } Hamming thresholds.

Next, we compare our method in the supervised scenario. Figure \ref{fig:mnist-prec-rec} shows the precision-recall of 50 bit codes for the MNIST dataset. The results are computed for the test set only, showing that our method outperforms the other methods in the more challenging out-of-sample scenario. Similarly, Figure \ref{fig:cifar-prec-rec} shows the comparison for the CIFAR-10 dataset.

To compare performance at different code lengths, we calculate the area under curve (AUC) for the precision-recall graph in the out-of-sample scenario. Table \ref{table:cifar-test-test} shows our results on the CIFAR-10 dataset, compared to the results of the spectral methods mentioned above. Our method consistently outperforms the other methods in both short and long codes.

%
%

Our solution for the signed min-cut problem includes an iterative scheme that continuously improves the initial guess. As mentioned before, we argue that the initial guess does not play a significant role in the final solution. In fact, at the end of the iterative process, an initial guess based on spectral methods provides similar results to a random initial guess. 

In the following experiment, we compute the codes only for the in-sample points (using a fixed random seed) and plot the loss as shown in Equation \ref{eq:lossL} at each code length. We show our results on the benchmark presented in \citeauthor{mlh} \shortcite{mlh} for six small datasets, consisting of 1000 training points.  Since we use the full versions of the MNIST and LabelMe datasets in the previous sections, we show here the four remaining datasets. 
We present the results generated from the following initial guesses: signed eigenvector corresponding to the smallest non-trivial eigenvalue of the Laplacian (L)  \cite{von2007tutorial}, signed Laplacian (SL)  \cite{kunegis2009slashdot}, the sign of the random projection of the 3 smallest non-trivial eigenvalues of the Laplacian  \cite{alon2004approximating} (AN),  and random guess (R). We show the effect of improving upon the initial guesses using bit update (BU) as presented in Algorithm \ref{alg2}. Results are shown in Figure \ref{fig:insample-loss}. The results clearly show that the random guess with bit update performs as well as or surpasses the costly spectral computations, while the bit update improves upon all of the initial guesses. 

A comparison between vector update (Algorithm \ref{alg1}) and bit update (Algorithm \ref{alg2}) for different initial guesses is shown in Figure \ref{fig:insample-loss-vubu}. It is clear that the bit update method outperforms the vector update. This is reasonable as the bit update is an optimization with smaller steps, allowing for a broader search for the optimum, whereas the vector update takes large steps and converges quickly into a local optimum.

\section{Conclusions}
We have shown a binary hashing method called Proximity Preserving Code (PPC) based on the signed graph-cut problem. We propose an approximation to this problem and show its advantages over other methods suggested in the literature. We also introduce a hashing framework that can work for both supervised and unsupervised datasets. The framework computes binary code that is more accurate than state-of-the-art graph hashing algorithms, especially in the challenging out-of-sample scenario. We believe the use of the signed graph problem instead of relaxation to the standard graph problem can prove beneficial in other algorithms as well.

\appendix
\section{Appendix}
\setcounter{theorem}{0}
The following four theorems are used in the Vector Update method in Algorithm \ref{alg1}. The proofs are provided below.

\begin{theorem} 
for any $n \times n$ matrix $W$ and $\b,\b' \in \{\pm 1\}^n$, ~if ~$\b'=sign(W \b)$,  then $\b'^T W \b \geq \b^T W \b$
\end{theorem}

\begin{proof}
Denote $\u=W\b$. Thus $\b'^T W \b = \b'^T \u = \sum_i \b'[i] \u[i]$. Since $\b'[i]\in \pm 1$, the maximal value is obtained when $\b'[i]=sign(\u[i])$. We arrive at $\b'^T W \b \geq \b^T W \b$.
\end{proof}

\begin{theorem}
Assuming  $ W $ is positive semidefinite (PSD), if ${\b'}^TW \b > \b^TW \b$, then 
$ {\b'}^TW{\b'}>\b^TW\b $.
\end{theorem}

\begin{proof}
We express vectors \(\b,{\b'}\) using the eigenvectors of W, \(W\u_i=\lambda_i \u_i\), \(\b=\sum\alpha_i \u_i\) and \({\b'}=\sum_i{\alpha'_i} \u_i\). Given
\begin{eqnarray}
{\b'}^TW\b=(\sum_i {\alpha'_i}\u_i^T)(\sum_i \alpha_i \lambda_i \u_i)= \\ 
= \sum_i {\alpha'_i}\alpha_i\lambda_i\geq\sum_i\alpha_i^2\lambda_i=\b^TW\b	
\end{eqnarray}

we would like to show that
\[{\b'}^T W{\b'} = \sum_i {\alpha'_i}^2\lambda_i\geq\sum_i\alpha_i^2\lambda_i =\b^TW\b\]

Using the PSD property of W (\(\lambda_i\geq0\)), we define 
$$\v=\begin{bmatrix}\alpha_1\sqrt{\lambda_1} \\ \vdots \\ \alpha_N\sqrt{\lambda_N}\end{bmatrix}\,{\v'}=\begin{bmatrix}{\alpha'_1}\sqrt{\lambda_1} \\ \vdots \\ {\alpha'_N}\sqrt{\lambda_N}\end{bmatrix} $$
Starting from our original assumption we have:
$ \v^T{\v'}\geq \v^T\v $. Using the triangular inequality we have:
$$\|\v\|\cdot\|{\v'}\|\geq \v^T {\v'}\geq \|\v\|^2$$
which follows that $ \|{\v'}\|\geq \|\v\|$ and accordingly 
$${\v'}^T {\v'} = {\b'}^T W{\b'} \geq {\b}^T W{\b} = \v^T\v $$
\end{proof}

\begin{theorem}
A symmetric matrix $W$ can become positive semidefinite by applying $W\leftarrow W+|\lambda| I$ where $\lambda$ is the smallest eigenvalue of $W$.
\end{theorem}
\begin{proof} According to the Gershgorin Circle Theorem \cite{gct}, 
for an $n \times n$ matrix W, define 
$R_i = \sum_{j=1,j\neq i}^n |w_{ij}|$. 
All eigenvalues of $W$ are in at least one of the disks $\{v: |v-w_{ii}| \leq R_i \}$.
Therefore, by adding the smallest eigenvalue to $w_{ii}$, the disks will only contain values greater than or equal to zero.
\end{proof}

\begin{theorem}
Adding a constant value to the diagonal of the weight matrix W will not affect the output code computed.
\end{theorem}
\begin{proof} In PPC, we optimize the vector $\b$ according to Equation \ref{eq:bWb}. Therefore:
\begin{gather*}
argmax_{\b} \b^T(W+|\lambda| I)\b \\
= argmax_{\b} \b^TW\b+|\lambda| \b^T\b \\
= argmax_{\b} \b^TW\b+|\lambda| n \\
= argmax_{\b} \b^TW\b 
\end{gather*}
\end{proof}

\bibliographystyle{aaai}
\bibliography{AAAI-LaviI.5779}
\end{document}